\documentclass[twoside]{article}

\usepackage{hyperref}
\usepackage[round]{natbib}
\usepackage{amsmath,amsfonts,amsthm,color,algorithm2e,enumitem,verbatim}
\usepackage{macros}

\usepackage[accepted]{aistats2016}
\begin{document}

%

%

\twocolumn[

\aistatstitle{No Regret Bound for Extreme Bandits}

\aistatsauthor{ Robert Nishihara \And David Lopez-Paz \And L\'eon Bottou }

\aistatsaddress{ UC Berkeley \\ \texttt{rkn@eecs.berkeley.edu} \And Facebook AI Research \\ \texttt{dlp@fb.com} \And Facebook AI Research \\ \texttt{leonb@fb.com}} ]

\begin{abstract}
  Algorithms for hyperparameter optimization abound, all of which work well under different and often unverifiable assumptions.
  Motivated by the general challenge of sequentially choosing which algorithm to use, we study the more specific task of choosing among distributions to use for random hyperparameter optimization.
  This work is naturally framed in the extreme bandit setting, which deals with sequentially choosing which distribution from a collection to sample in order to minimize (maximize) the single best cost (reward). 
  Whereas the distributions in the standard bandit setting are primarily characterized by their means, a number of subtleties arise when we care about the minimal cost as opposed to the average cost. 
  For example, there may not be a well-defined ``best'' distribution as there is in the standard bandit setting.
  The best distribution depends on the rewards that have been obtained and on the remaining time horizon.
  Whereas in the standard bandit setting, it is sensible to compare policies with an oracle which plays the single best arm, in the extreme bandit setting, there are multiple sensible oracle models. 
  We define a sensible notion of ``extreme regret'' in the extreme bandit setting, which parallels the concept of regret in the standard bandit setting. 
  We then prove that no policy can asymptotically achieve no extreme regret.
\end{abstract}

\section{Introduction}

Our motivation comes from hyperparameter optimization and more generally from the challenge of minimizing a black-box objective~$f \colon \Omega \to [0,1]$ which we can only evaluate pointwise.
As an example,~$\omega \in \Omega$ may parameterize the architecture of a convolutional network, and~$f(\omega)$ may be the validation error when the network with that architecture is trained on a particular data set.
A number of approaches have been applied to the optimization of~$f$ including Bayesian optimization, covariance matrix adaptation, random search, and a variety of other methods (for an incomplete list, see \citet{bergstra2012random,bergstra2011algorithms,snoek2012practical,hansen2006cma,wang2013bayesian,lagarias1998convergence,powell2006newuoa,duchi2015optimal}).

In some sense, random search is the benchmark of choice.
Whereas other approaches work well under various and often unverifiable conditions (such as smoothness or convexity of the objective), random search has strong finite-sample guarantees that hold without any assumptions on the function under consideration. 
This guarantee is illustrated by the so-called {\em rule of 59},\footnote{Though they are known, the rule of 59 and \lemref{lem:random_guarantee} do not appear in \citet{bergstra2012random}, and they are difficult to find in the literature.} which states that the best of~$59$ random samples will be in the best~$5$ percent of all samples with probability at least~$0.95$. 
More generally, any distribution over the set of hyperparameters $\Omega$ induces a distribution $\mu$ over the validation error in $[0,1]$.
Let $F_{\mu}$ be the cumulative distribution function of $\mu$, and suppose that $F_{\mu}$ is continuous.
Suppose that $x_1,\ldots,x_T$ are independent and identically-distributed samples from $\mu$ (obtained, for instance, by independently sampling hyperparameters $\omega_t$ and evaluating $x_t=f(\omega_t)$ for $1 \le t \le T$). 
The following is known. 
\begin{lemma} \label{lem:random_guarantee}
  The distribution of the extreme cost $\min\{x_1,\ldots,x_T\}$ is easily described with quantiles. 
  We have $P(F_{\mu}(\min\{x_1,\ldots,x_T\}) \le \alpha) = 1 - (1-\alpha)^T$. 
  More specifically, $F_{\mu}(\min\{x_1,\ldots,x_T\})$ is a~$\text{Beta}(1,T)$ random variable. 
\end{lemma}
\begin{proof}
  The event $F_{\mu}(\min\{x_1,\ldots,x_T\}) > \alpha$ happens if and only if $F_{\mu}(x_t) > \alpha$ for each $t$, which happens with probability $(1-\alpha)^T$. 
Differentiating the resulting cumulative distribution function gives the density function of a $\text{Beta}(1,T)$ random variable.
\end{proof}

The generality of \lemref{lem:random_guarantee} comes at a price. 
The guarantee is given with respect to the distribution $\mu$, but there is no guarantee about $\mu$ itself.
Different induced distributions $\mu$ may arise from different parameterizations of the hyperparameter space $\Omega$ (for example, from the decision to put a uniform or a log-uniform distribution over a coordinate of $\omega$), and the allocation of mass over $[0,1]$ may vary wildly based on these choices.

Furthermore, the flip side of making no assumptions on the underlying objective is that random search fails to adapt to easy problems. 
When the objective under consideration satisfies various regularity conditions (as real-world objectives often do), more heavily-engineered approaches will likely outperform random search. 
That said, it is not clear how to know that a given algorithm is outperforming random search without also running random search.
For this reason, the benefits of a potentially faster algorithm are blunted when one must also run the slow algorithm to verify the performance of the fast algorithm.

Given the variety of existing hyperparameter optimization algorithms, it would be desirable to devise a strategy for sequentially choosing which algorithm to use in a way that performs nearly as well as if we had only used the single best algorithm.
We consider the simpler problem of choosing which of several distributions over hyperparameters to use for random search.
In \thmref{thm:main_result}, we show that even in this simplified setting, no strategy guarantees performance that is asymptotically as good as the single best distribution, at least not without stronger assumptions. 

We will frame our negative result in the extreme bandit setting \citep{carpentier2014extreme}, also called the max~$K$-armed bandit setting \citep{cicirello2005max}. 
Prior work has focused on designing algorithms that perform asymptotically as well as the single best distribution under parametric (or semiparametric) assumptions on the possible distributions \citep{cicirello2005max,carpentier2014extreme}. 
Instead, we focus on probing the difficulty of the problem, pointing out a number of subtleties that arise in this setting that do not show up in the conventional bandit setting.

\section{The Extreme Bandit Setting}

\citet{cicirello2005max} introduce the extreme bandit problem (they call it the max $K$-armed bandit problem) as follows.
We are given a tuple of unknown distributions (arms)~$\mu_1^K=(\mu_1,\ldots,\mu_K)$. 
The~$k$th distribution generates sample~$x_{k,t}$ at time~$t$, for integer $t \ge 1$, and all of the samples $x_{k,t}$ are independent.
A policy~$\pi$ is a function that, at each time~$t$, chooses the index $k_t$ of a distribution to sample based on the previously observed samples.
That is,
\begin{equation*}
  k_t = \pi( \! \underbrace{k_1,\ldots,k_{t-1}}_{\text{past arm choices}},\underbrace{x_{k_1,1},\ldots,x_{k_{t-1},t-1}}_{\text{past values}})  .
\end{equation*}
We would like to compare the performance of a policy $\pi$ to that of an oracle policy $\pi_*$ that has access to knowledge of the distributions $\mu_1^K$, so
\begin{equation*}
  k_t^* = \pi_*(\mu_1^K, k_1^*,\ldots,k_{t-1}^*,x_{k_1^*,1},\ldots,x_{k_{t-1}^*,t-1}) .
\end{equation*}
Both \citet{cicirello2005max} and \citet{carpentier2014extreme} phrase their results in terms of the maximization of a reward rather than the minimization of a cost. 
They define the ``regret'' of policy $\pi$ with respect to the oracle $\pi_*$ over a time horizon of~$T$ as
\begin{equation*}
  G_T^{\pi,\pi_*} = \mathbb E\left[ \max_{t \le T} x_{k_t^*,t} \right] - \mathbb E\left[ \max_{t \le T} x_{k_t,t} \right] .
\end{equation*}
Under semiparametric assumptions on $\mu_1^K$, \citet{carpentier2014extreme} exhibit a policy $\pi$ such that 
\begin{equation} \label{eq:old_regret_def}
  G_T^{\pi,\pi_*} \,\,\text{is}\,\,o\left(\mathbb E\left[ \max_{ t \le T} x_{k_t^*,t} \right]\right) 
\end{equation}
or equivalently,
\begin{equation} \label{eq:old_regret_def_2} 
\lim_{T \to \infty} \frac{\mathbb E\left[ \max_{ t \le T} x_{k_t,t} \right]}{\mathbb E\left[ \max_{ t \le T} x_{k_t^*,t} \right]} \to 1.
\end{equation}
The result in \eqref{eq:old_regret_def} is superficially similar to results in the standard bandit setting.
However, while the condition in \eqref{eq:old_regret_def} is sensible for the setting considered by \citet{carpentier2014extreme} (where the distributions $\mu_1^K$ have unbounded support), it is particularly sensitive to the nature of the distributions.
For instance, the result in \eqref{eq:old_regret_def} is trivially achieved when the distributions have bounded support (for example, when the support is contained in $[0,1]$ as in hyperparameter optimization).
In this case, both the numerator and denominator converge to the upper bound of the support and $G_T^{\pi,\pi_*} \to 0$ (for any policy that chooses each distribution infinitely often). 

Furthermore, the condition in \eqref{eq:old_regret_def_2} is asymmetric with respect to maximization and minimization. 
When performing minimization of a cost instead of maximization of a reward (using distributions supported in $[0,1]$), both $\mathbb E\left[ \min_{ t \le T} x_{k_t,t} \right]$ and $\mathbb E\left[ \min_{ t \le T} x_{k_t^*,t} \right]$ may approach $0$, in which case the ratio may exhibit radically different behavior.
In \exref{ex:geom} and \exref{ex:unif}, we demonstrate some of the peculiarities of this performance metric in the minimization setting. 

\begin{example} \label{ex:geom}
  Suppose~$\mu_1$ is a Bernoulli distribution with mean parameter~$0<p<1$ and suppose that~$\mu_2$ is a point mass on~$1$. 
  Consider a policy~$\pi$ which chooses~$\mu_2$ at~$t=1$ and then chooses~$\mu_1$ for all~$t > 1$ and a policy $\pi_*$ which always chooses $\mu_1$. 
  We have
  \begin{equation*}
    \lim_{T \to \infty} \frac{\mathbb E\left[ \min_{ t \le T} x_{k_t,t} \right]}{\mathbb E\left[ \min_{ t \le T} x_{k_t^*,t} \right]} = \lim_{T \to \infty} \frac{p^{T-1}}{p^T} = \frac{1}{p} ,
  \end{equation*}
  which remains bounded away from~$1$ even though the policy~$\pi$ acted optimally at every time step after~$t=1$. 
\end{example}

\begin{example} \label{ex:unif}
  Suppose~$\mu_1$ is the uniform distribution over~$[0,1]$ and suppose that~$\mu_2$ is a point mass on~$1$.
  Consider a policy~$\pi$ which chooses~$\mu_2$ at~$t=1$ and then chooses~$\mu_1$ for all~$t > 1$ and a policy $\pi_*$ which always chooses $\mu_1$.
  We have
  \begin{equation*}
    \lim_{T \to \infty} \frac{\mathbb E\left[ \min_{t \le T} x_{k_t,t} \right]}{\mathbb E\left[ \min_{t \le T} x_{k_t^*,t} \right]} = \lim_{T \to \infty} \frac{T^{-1}}{(T+1)^{-1}} \to  1.
  \end{equation*}
  Note above that the minimum of~$T$ independent uniform random variables is a~$\text{Beta}(1,T)$ random variable, which has mean $1/(T+1)$. 
\end{example}

Despite the fact that the policy $\pi$ acts optimally at every time step other than $t=1$ in both \exref{ex:geom} and \exref{ex:unif}, the ratios of their expectations to that of the oracle $\pi_*$ exhibit wildly different behaviors. 

To avoid this sensitivity, we define ``extreme regret'' as follows.
\begin{definition} \label{def:regret}
  We define the extreme regret of the policy $\pi$ with respect to the oracle policy $\pi_*$ over a time horizon of $T$ as
  \begin{equation*}
    R_T^{\pi,\pi_*} = \frac{1}{T} \min_{T' \ge 1} \left\{ T' \,:\, \mathbb E\left[ \min_{t \le T'} x_{k_t,t} \right] \le \mathbb E\left[ \min_{t \le T} x_{k_t^*,t} \right] \right\} .
  \end{equation*}
  Note that $R_T^{\pi,\pi_*}$ depends on the tuple of distributions $\mu_1^K$, but we suppress this dependence in our notation. 
\end{definition}
Then $R_T^{\pi,\pi_*}$ is essentially the ratio of the time horizons $T'$ to $T$ over which the policy $\pi$ and the oracle $\pi_*$ perform equally well. 
This definition is sensible regardless of whether the samples are bounded or unbounded, whether we care about minimization or maximization, and regardless of how we scale or translate the distributions.
Note that in both \exref{ex:geom} and \exref{ex:unif}, we have $R_T^{\pi,\pi_*}=\frac{T+1}{T} \to 1$. 
Despite its apparent difference, as we discuss in \secref{sec:analogy}, \defref{def:regret} is closely related to the notion of regret used in the standard bandit setting. 

\begin{definition} \label{def:no_regret}
  We say that policy $\pi$ achieves ``no extreme regret'' with respect to the oracle $\pi_*$ if $\limsup_T R_T^{\pi,\pi_*} \le 1$ for all tuples of distributions $\mu_1^K$. 
\end{definition}
\defref{def:no_regret} is fairly lenient. 
Had we defined ``no extreme regret'' using the condition given in \eqref{eq:old_regret_def}, our main result in \thmref{thm:main_result} could have been made even stronger, but we view that as undesirable as illustrated by \exref{ex:geom} and \exref{ex:unif}.
Moreover, the quantities in \defref{def:regret} and \defref{def:no_regret} closely parallel quantities of interest in the standard bandit setting, as we show in \secref{sec:analogy}.

\subsection{Analogy with the Standard Bandit Setting} \label{sec:analogy}

\defref{def:regret} and \defref{def:no_regret} parallel the intuition of the standard bandit setting, which (when minimizing a cost) studies the rate of convergence of
\begin{equation} \label{eq:regret_ratio_standard_bandit}
  \frac{\mathbb E\left[ \sum_{t=1}^{T} x_{k_t,t} \right]  - \min_k \mathbb E\left[ \sum_{t=1}^{T} x_{k,t} \right]}{\min_k \mathbb E\left[ \sum_{t=1}^{T} x_{k,t} \right]} \to 0 .
\end{equation}
Adding $1$ to both sides, this is the same as studying the rate of convergence of 
\begin{equation*}
  \frac{\mathbb E\left[ \sum_{t=1}^{T} x_{k_t,t} \right]}{T \min_k \mathbb E[x_{k,t}]} \to 1 .
\end{equation*}
Now, observe that we have
  \begin{align} \label{eq:standard_bandit_analogy}
  & \,\, \frac{\mathbb E\left[ \sum_{t=1}^{T} x_{k_t,t} \right]}{T \min_k \mathbb E[x_{k,t}]} \nonumber \\
    \approx & \,\, \frac{1}{T} \min_{T' \ge 1} \left\{ T' \,:\, \frac{\mathbb E\left[ \sum_{t=1}^T x_{k_t,t} \right]}{\min_k \mathbb E[  x_{k,t} ]} \le T' \right\} \\
  = & \,\, \frac{1}{T} \min_{T' \ge 1} \left\{ T' \,:\, \mathbb E\Bigg[ \sum_{t=1}^T x_{k_t,t} \Bigg] \le \min_k \mathbb E\Bigg[ \sum_{t=1}^{T'} x_{k,t} \Bigg] \right\} , \nonumber 
\end{align}
which is essentially the ratio of the time horizons over which the policy and the oracle perform equally well. 
The two sides of the approximate equality in \eqref{eq:standard_bandit_analogy} differ by at most $1/T$. 
In the standard bandit setting, the term ``regret'' often refers to the numerator in \eqref{eq:regret_ratio_standard_bandit} and not the quantity in \eqref{eq:standard_bandit_analogy}. 
However, as the above computation shows, the two quantities are closely related, and they capture the same phenomenon. 
We will phrase our results in terms of the quantity $R_T^{\pi,\pi_*}$ from \defref{def:regret}, which parallels the quantity in \eqref{eq:standard_bandit_analogy}.

\section{Oracle Models}

In the standard multi-armed bandit setting, if an oracle with knowledge of the distributions of the arms seeks to minimize the expected sum of the losses, it should simply choose to play the arm with the lowest mean.
This is true regardless of the time horizon.
By analogy with the usual multi-armed bandit setting, \citet{cicirello2005max} and \citet{carpentier2014extreme} both consider the oracle policy in \defref{def:single_armed_oracle} that plays the single ``best'' arm. 
\begin{definition}[single-armed oracle] \label{def:single_armed_oracle}
  The single-armed oracle is the oracle, which over a time horizon of~$T$, plays the single best arm
  \begin{equation*}
    \argmin_k \mathbb E \left[\min_{ t \le T} x_{k,t}\right] .
  \end{equation*}
\end{definition}
The single-armed oracle provides a good benchmark for comparison, but it is not the optimal oracle policy.
When the time horizon is known in advance, the optimal oracle policy is given in \defref{def:optimal_oracle}.
\begin{definition}[optimal oracle] \label{def:optimal_oracle}
  The optimal oracle over a time horizon of~$T$ plays the policy that solves
  \begin{equation*}
    \argmin_{\pi} \mathbb E \left[ \min_{ t \le T} x_{k_t,t} \right] .
  \end{equation*}
\end{definition}
When the time horizon is not known in advance, one possible oracle strategy is a greedy strategy given in \defref{def:greedy_oracle}.
\begin{definition}[greedy oracle] \label{def:greedy_oracle}
  The greedy oracle chooses the arm~$k_t^*$ at time~$t$ that gives the maximal expected improvement over the current best value~$y_{t-1}=\min_{ s \le t-1} x_{k_s^*,s}$.
  That is,
  \begin{equation*}
    k_t^* = \argmin_k \mathbb E \left[ \min\{x_{k,t}, y_{t-1}\} \,|\, x_{k_1^*,1}, \ldots, x_{k_{t-1}^*,t-1} \right] . 
  \end{equation*}
\end{definition}
Unlike the greedy oracle, both the single-armed oracle and the optimal oracle require knowledge of the time horizon. 
Indeed, as shown in \exref{ex:different_arms}, the notion of a ``best'' arm is not well-defined outside of a specific time horizon. 
The best arm depends on the time horizon.
This point contrasts sharply with the usual multi-armed bandit setting. 
\begin{example} \label{ex:different_arms}
  Suppose we have an infinite collection of arms~$\mu_s$ indexed by~$0<s<1$.
  Let~$x_{s,t}$ be a sample from~$\mu_s$ and suppose that~$P(x_{s,t}=s)=s$ and~$P(x_{s,t}=1)=1-s$.
  Then the optimal $s$ is $\Theta((\log T)/T)$. 
\end{example}
We elaborate on \exref{ex:different_arms} in \appref{sec:ex_different_arms}. 
One difference between the single-armed oracle and the optimal oracle is that the optimal oracle can adapt its strategy based on the samples that it receives, whereas the single-armed oracle is non-adaptive.
Its strategy is fixed ahead of time. 
\exref{ex:mixed_strategy} shows that the single-armed oracle is not even the optimal non-adaptive oracle.
A mixed strategy may outperform any policy that plays only a single arm. 
\begin{example} \label{ex:mixed_strategy}
  Consider a time horizon~$T=2$ and consider two arms.
  Suppose that samples~$x_{1,t}$ from~$\mu_1$ deterministically equal~$1/2$ and that samples~$x_{2,t}$ from~$\mu_2$ satisfy~$P(x_{2,t}=0)=1/4$ and~$P(x_{2,t}=1)=3/4$.
  Then
  \begin{align*}
    & \mathbb E \min_{1 \le t \le 2} x_{1,t} = \frac12 \\ 
    & \mathbb E \min_{1 \le t \le 2} x_{2,t} = \frac{9}{16} \\
    & \mathbb E \min \{ x_{1,1}, x_{2,2} \} = \frac38 .
  \end{align*}
  This example shows that a fixed strategy that plays both arms can outperform any policy that plays a single-arm. 
\end{example}

We described three different oracle models above. 
One caveat is that in the event that there is a well-defined best arm, that is, some arm~$k_*$ such that~$P(x_{k_*,t} \le \alpha) \ge P(x_{k,t} \le \alpha)$ for all~$k$ and all~$0 \le \alpha \le 1$, then these three oracles all coincide and we need not worry about which oracle to use for comparison. 
This is roughly the case in prior work. 
\citet{cicirello2005max} and \citet{carpentier2014extreme} make (semi)parametric assumptions on the distributions of the arms which essentially restrict the setting to have a well-defined best arm. 

Despite the fact that the single-armed oracle is not the optimal oracle strategy, it is often a sufficiently strong baseline for measuring the performance of our policies.
When we cannot even do as well as the single-armed oracle, as will be the case in \thmref{thm:main_result}, then we also cannot do as well as the optimal oracle. 
For the remainder of the paper, we will compare to the single-armed oracle.
However, the results necessarily hold for comparisons to the optimal oracle as well. 

\section{Main Result} \label{sec:main_result}

\thmref{thm:main_result} shows that no policy can be guaranteed to perform asymptotically as well as the single best distribution. 
That is, it is impossible to achieve ``no extreme regret'' in the extreme bandit problem. 
This result contrasts sharply with results in the standard bandit setting, where it is possible to achieve no regret under relatively mild conditions on the distributions $\mu_1^K=(\mu_1,\ldots,\mu_K)$. 

\begin{theorem} \label{thm:main_result}
  For any policy~$\pi$, there exist distributions $\mu_1^K$ such that $\limsup_T R_T^{\pi,\pi_*} \ge K$, where $\pi_*$ is the single-armed oracle.
\end{theorem}

We prove \thmref{thm:main_result} in \secref{sec:proof_of_main_result}. 
The main components of the proof are \lemref{lem:upper_bound_oracle}, which upper bounds the performance of the single-armed oracle and \lemref{lem:lower_bound_policy}, which lower bounds the performance of the policy $\pi$. 

This result shows that the extreme bandit problem is fundamentally different from the standard multi-armed bandit problem, where a variety of policies perform asymptotically as well as the single best arm.
Indeed, in the standard bandit problem, the arms are primarily characterized by their means, and so it suffices to estimate the means of the arms and play the best one.
However, as discussed in \exref{ex:different_arms}, there is no well-defined best arm in the extreme bandit problem.
Our construction will create a situation where the ``best'' arm periodically switches among the $K$ distributions so that the policy $\pi$ often ends up choosing the ``wrong'' arm. 

For $i \ge 1$, let $\alpha_i=(8K)^{-(i!)^2}$\!. 
Our construction will involve a sum of point masses at the values $\alpha_i$. 
It is easily verified that the sequence $\alpha_i$ satisfies the conditions in \lemref{lem:alpha_conditions}. 
\begin{lemma} \label{lem:alpha_conditions}
The sequence $\alpha_i$ satisfies the following properties.
\begin{enumerate}[label=(\Alph*)]
\item $\sum_{j=1}^{\infty} \alpha_j \le 1/2$ \label{ass:1} 
\item $\alpha_i \le \frac{1}{4(1+i)}$ \label{ass:2}
\item $\sum_{j=i+1}^{\infty} \alpha_j \le \frac{\alpha_i}{i K}$ \label{ass:3}
\item $\alpha_i \le \alpha_{i-1}^i 2^{-i}$. \label{ass:4}
\end{enumerate}
\end{lemma}
Henceforth, we will not need the exact values of the sequence, we will only need the properties enumerated in \lemref{lem:alpha_conditions}. 
For $b=(b_1,b_2,\ldots) \in \{1,\ldots,K\}^{\infty}$, define the tuple of distributions $\mu_1^K(b)=(\mu_1(b),\ldots,\mu_K(b))$ via
\begin{equation*}
  \mu_k(b) = \gamma_k(b) \delta_1 + \sum_{i = 1}^{\infty} \ind[b_i=k] \, \alpha_i \, \delta_{\alpha_i} 
\end{equation*}
where
\begin{equation*}
  \gamma_k(b) = 1-\sum_{i=1}^{\infty} \ind[b_i=k]\alpha_i .
\end{equation*}
Here, $\delta_c$ represents a point mass at~$c$,~$\ind[\xi]$ is the~$\{0,1\}$-indicator function of the event~$\xi$, and~$\gamma_k(b)$ is chosen to make~$\mu_k(b)$ a probability measure.
Let $M_K$ be the set of tuples of distributions that can be obtained in this way.
The value $b_i$ simply assigns the point mass $\delta_{\alpha_i}$ to one of the $K$ distributions. 
We let $D$ denote the distribution over the set $\{1,\ldots,K\}^{\infty}$ defined so that the $b_i$'s are independent uniform random variables in $\{1,\ldots,K\}$. 

Define the time horizon $T_i = \lceil \log(1/\alpha_i)/\alpha_i \rceil$. 
Instead of controlling $R_T^{\pi,\pi_*}$ for every $T$, we will control the quantity specifically for the time horizons $T_i$. 
In our construction, the $b_i$th arm in the tuple will be the best arm over the time horizon $T_i$, and the other arms will be substantially worse. 
We will show that, for a fixed $i$, we can construct a tuple $\mu_1^K$ so that the policy $\pi$ takes roughly $K$ times longer than the single-armed oracle $\pi_*$ to obtain the value $\alpha_i$ (that is, $\pi_*$ requires roughly $T_i$ samples and $\pi$ requires roughly $T_i' \approx K T_i$ samples). 
Using the probabilistic method, we will then show that we can find a tuple $\mu_1^K$ so that the policy takes roughly $K$ times longer than the oracle to obtain the value $\alpha_i$ for infinitely many values of $i$. 

\subsection{Upper Bound on Oracle Performance}

We begin by giving an upper bound on the performance of the oracle policy that plays the single best arm over the time horizon $T_i$.
This bound holds uniformly over $M_K$.

\begin{lemma} \label{lem:upper_bound_oracle}
  Suppose that $\mu_1^K(b) \in M_K$. 
  If $\pi_*$ is the single-armed oracle from \defref{def:single_armed_oracle}, then
  \begin{equation*}
    \mathbb E\left[ \min_{t \le T_i} x_{k_*,t} \right] < 2\alpha_i .
  \end{equation*}
\end{lemma}
\begin{proof}
  Recall that $b_i$ is the index of the distribution that has a point mass at $\alpha_i$.
  We have
  \begin{equation*}
    \mathbb E\left[ \min_{t \le T_i} x_{k_*,t} \right] = \min_k \mathbb E\left[ \min_{t \le T_i} x_{k,t} \right] \le \mathbb E\left[ \min_{t \le T_i} x_{b_i,t} \right] .
  \end{equation*}
  The term on the right hand side can be rewritten as
  \begin{align*}
    & \,\, \mathbb E\left[ \ind\left[ \min_{t \le T_i} x_{b_i,t} \le \alpha_i \right] \min_{t \le T_i} x_{b_i,t}  \right] \\
    & \qquad + \mathbb E\left[ \ind \left[ \min_{t \le T_i} x_{b_i,t} > \alpha_i \right] \min_{t \le T_i} x_{b_i,t} \right] \\
    \le & \,\, \alpha_i P\left[ \min_{t \le T_i} x_{b_i,t} \le \alpha_i \right] + P\left[ \min_{t \le T_i} x_{b_i,t} > \alpha_i \right] \\
    \le & \,\, \alpha_i + P\left[ \min_{t \le T_i} x_{b_i,t} > \alpha_i \right] .
  \end{align*}
  The first inequality follows by upperbounding the term $\min_{t \le T_i} x_{b_i,t}$ by $\alpha_i$ in the first term and by $1$ in the second term. 
  The second inequality follows by upperbounding the first probability by $1$. 
  To finish the lemma, note that
  \begin{equation*}
    P\left[ \min_{t \le T_i} x_{b_i,t} > \alpha_i \right] \le (1-\alpha_i)^{T_i} <  e^{-\alpha_i T_i} \le \alpha_i ,
  \end{equation*}
  where the third inequality uses the definition $T_i=\lceil \log(1/\alpha_i)/\alpha_i\rceil$. 
\end{proof}

\subsection{Lower Bound on Performance of $\pi$}

Here, we give a lower bound on the performance of any fixed policy $\pi$, when averaged over a collection of tuples of distributions. 

Define the time horizon $T_i'=\lfloor c_i K \log(1/\alpha_i)/\alpha_i \rfloor$, where $c_i= (1-1/i)/((1+1/i)^2+2/i)$.
The constant $c_i$ is a correction term that converges to $1$ as $i \to \infty$.
Its specific value is not meaningful. 
The goal of this section is roughly to show that the performance of the policy $\pi$ over a time horizon of $T_i'$ is comparable to the performance of the oracle policy over a time horizon of $T_i$.

Throughout this section, we will fix an index $i$ and we fix $b_j$ for all $j \ne i$.
Then we define the sequence $b^{k'}=(b_1^{k'},b_2^{k'},\ldots)$ via $b_j^{k'}=b_j$ for $j \ne i$ and $b_i^{k'}=k'$.
The $K$ tuples $\mu_1^K(b^{k'})$ for different values of $k'$ are identical in all respects except for the index of the distribution that possesses the point mass $\delta_{\alpha_i}$ and the amount of mass $\gamma_k(b^{k'})$ that the $k$th distribution in the $k'$th tuple assigns to $\delta_1$. 

Define the tuple of distributions $\eta_1^K(\overline{b})=(\eta_1(\overline{b}),\ldots,\eta_K(\overline{b}))$ by $\eta_k(\overline{b})=\frac{1}{K}\sum_{k'=1}^K \mu_k(b^{k'})$.
Let $\gamma_k(\overline{b}) := \frac{1}{K} \sum_{k'=1}^K \gamma_k(b^{k'})$ denote the probability that $\eta_k(\overline{b})$ assigns to the value $1$.
The tuple $\eta_1^K(\overline{b})$ is the average of the tuples $\mu_1^K(b^{k'})$ over the different values of $k'$.

We begin with \lemref{lem:compare_probs} which compares the probability that policy $\pi$ obtains the value $\alpha_i$ when averaged over the tuples $\mu_1^K(b^{k'})$ with the probability that $\pi$ obtains the value $\alpha_i$ in the tuple $\eta_1^K(\overline{b})$.
This comparison is helpful because each distribution in the tuple $\eta_1^K(\overline{b})$ assigns the same mass of $\alpha_i/K$ to $\alpha_i$ and so the probability that $\pi$ obtains $\alpha_i$ when run on the tuple $\eta_1^K(\overline{b})$ does not depend on $\pi$ (it is simply $(1-\alpha_i/K)^{T_i'}$ where $T_i'$ is the time horizon).
Of course, as stated, we are actually concerned with the probability that $\pi$ obtains a value less than or equal to $\alpha_i$, but because of \lemsubref{lem:alpha_conditions}{ass:3}, the contribution of the smaller terms will not be too great. 
\begin{lemma} \label{lem:compare_probs}
  We have
  \begin{align*}
    & \,\, \frac{1}{K} \sum_{k'=1}^K P\left[ \min_{t \le T_i'} x_{k_t,t} \ge \alpha_{i-1} \,\middle|\, \mu_1^K(b^{k'}) \right] \\
    \ge & \,\,  c P\left[ \min_{t \le T_i'} x_{k_t,t} \ge \alpha_{i-1} \,\middle|\, \eta_1^K(\overline{b}) \right],
  \end{align*}
  where $c=e^{-\frac{2\alpha_i T_i'}{iK}}$. 
  In our notation, we condition on $\mu_1^K(b^{k'})$ to indicate the tuple of distributions being used. 
\end{lemma}
\begin{proof}
  Define $S(\pi,\mu_1^K,T)$ to be the set of actions and values that can be obtained by following policy $\pi$ on the tuple $\mu_1^K$ for a time horizon of $T$.
  That is,
  \begin{align*}
    & \,\, S(\pi,\mu_1^K,T) \\
    = & \,\, \left\{ (k_t,x_t)_{t=1}^T \,:\, 
    \begin{array}{cc}
      k_t=\pi(k_1,\ldots,k_{t-1},x_1,\ldots,x_{t-1}) \\
      x_t \in \supp(\mu_{k_t}) 
      \end{array}
    \right\} ,
  \end{align*}
  where $\supp(\mu_{k_t})$ is the support of the distribution $\mu_{k_t}$. 
  Then define $S(\pi,\mu_1^K,T,i)$ to be the subset of $S(\pi,\mu_1^K,T)$ such that all values are greater than or equal to $\alpha_{i-1}$.
  That is,
  \begin{equation*}
    S(\pi,\mu_1^K,T,i) = \left\{ (k_t,x_t)_{t=1}^T \in S(\pi,\mu_1^K,T) \,:\, x_t \ge \alpha_{i-1}  \right\} .
  \end{equation*}
  Critically, note that
\begin{equation}\label{eq:sets_equal}
  \begin{aligned} 
    S(\pi,\eta_1^K(\overline{b}),T_i',i) & = S(\pi,\mu_1^K(b^{1}),T_i',i) \\
    & \,\,\,\vdots  \\
    & = S(\pi,\mu_1^K(b^{K}),T_i',i) .
  \end{aligned}
\end{equation}
  \eqref{eq:sets_equal} holds because the supports of the tuples $\mu_1^K(b^{k'})$ and $\eta_1^K(\overline{b})$ only differ on $\alpha_i$, but we are considering only values that are at least $\alpha_{i-1}$, so this difference does not affect the sets. 
  We shall refer to this common set as $S$. 
  We have 
  \begin{align} \label{eq:expand_prob_as_sum}
    & \,\, P\left[ \min_{t \le T_i'} x_{k_t,t} \alpha_{i-1} \,\middle|\, \mu_1^K(b^{k'}) \right] \\
    = & \,\, \sum_{S} \left( \prod_{j=1}^{i-1} \alpha_j^{|\{t \,:\, x_t=\alpha_j\}|} \prod_{k=1}^K \gamma_k(b^{k'})^{|\{t \,:\, k_t=k,x_t=1\}|} \right) .
  \end{align}
  It follows that
  \begin{equation} \label{eq:compare_probs_comp_1}
  \begin{aligned}
    & \,\, \frac{1}{K} \sum_{k'=1}^K P\left[ \min_{t \le T_i'} x_{k_t,t} \ge \alpha_{i-1} \,\middle|\, \mu_1^K(b^{k'}) \right] \\
    = & \,\, \frac{1}{K} \sum_{k'=1}^K \sum_{S} \left( \prod_{j=1}^{i-1} \alpha_j^{|\{t \,:\, x_t=\alpha_j\}|} \prod_{k=1}^K \gamma_k(b^{k'})^{|\{t \,:\, k_t=k,x_t=1\}|} \right) \\
    = & \,\,  \sum_{S} \left( \prod_{j=1}^{i-1} \alpha_j^{|\{t \,:\, x_t=\alpha_j\}|} \left(\frac{1}{K} \sum_{k'=1}^K \prod_{k=1}^K \gamma_k(b^{k'})^{|\{t \,:\, k_t=k,x_t=1\}|} \right) \right) ,
  \end{aligned}
  \end{equation}
  where the first equality uses \eqref{eq:expand_prob_as_sum} and the second equality simply rearranges the terms. 
  We would like to essentially apply Jensen's inequality to say something like
  \begin{equation} \label{eq:hypothetical_jensen}
    \frac{1}{K} \sum_{k'=1}^K \prod_{k=1}^K \gamma_k(b^{k'})^{|\{t \,:\, k_t=k,x_t=1\}|} \ge \prod_{k=1}^K \gamma_k(\overline{b})^{|\{t \,:\, k_t=k,x_t=1\}|} .
  \end{equation}
  Unfortunately, despite the fact that $\gamma_k$ is convex on the relevant region, $\prod_{k=1}^K \gamma_k$ is not quite convex.
  However, it is nearly convex, and as we show in \lemref{lem:compare_mixture_to_average}, \eqref{eq:hypothetical_jensen} holds up to a correction factor of $e^{-\frac{2\alpha_i T_i'}{iK}}$. 
  Using this result in \eqref{eq:compare_probs_comp_1} gives 
  \begin{align*}
    & \,\, \frac{1}{K} \sum_{k'=1}^K P\left[ \min_{t \le T_i'} x_{k_t,t} \ge \alpha_{i-1} \,\middle|\, \mu_1^K(b^{k'}) \right] \\
    \ge & \,\, e^{-\frac{2\alpha_i T_i'}{iK}} \sum_{S} \left( \prod_{j=1}^{i-1} \alpha_j^{|\{t \,:\, x_t=\alpha_j\}|}  \prod_{k=1}^K \gamma_k(\overline{b})^{|\{t \,:\, k_t=k,x_t=1\}|}  \right) \\
    = & \,\, e^{-\frac{2\alpha_i T_i'}{iK}} P\left[ \min_{t \le T_i'} x_{k_t,t} \ge \alpha_{i-1} \,\middle|\, \eta_1^K(\overline{b}) \right] .
  \end{align*}
  the first inequality uses \lemref{lem:compare_mixture_to_average} and the last equality holds for the same reason that \eqref{eq:expand_prob_as_sum} holds.
\end{proof}

In \lemref{lem:lower_bound_policy}, we turn the bound in \lemref{lem:compare_probs} on the probability of obtaining $\alpha_i$ into a bound on the performance of $\pi$. 
Note that \lemref{lem:lower_bound_policy} holds uniformly over the values of $b_j$ for $j \ne i$. 
\begin{lemma} \label{lem:lower_bound_policy}
  We have
  \begin{equation*}
    \frac{1}{K} \sum_{k'=1}^K \mathbb E\left[ \min_{t \le T_i'} x_{k_t,t} \,\middle|\,  \mu_1^K(b^{k'}) \right] \ge 2\alpha_i .
  \end{equation*}
\end{lemma}
\begin{proof}
We have
\begin{equation} \label{eq:lem_lower_bound_policy_comp_1}
  \begin{aligned}
    & \,\, \frac{1}{K} \sum_{k'=1}^K \mathbb E\left[ \min_{t \le T_i'} x_{k_t,t} \,\middle|\,  \mu_1^K(b^{k'}) \right] \\
    \ge & \,\, \frac{\alpha_{i-1}}{K} \sum_{k'=1}^K P\left[ \min_{t \le T_i'} x_{k_t,t} \ge \alpha_{i-1} \,\middle|\,  \mu_1^K(b^{k'}) \right]  \\
    \ge & \,\, \alpha_{i-1} e^{-\frac{2\alpha_i T_i'}{iK}} P\left[ \min_{t \le T_i'} x_{k_t,t} \ge \alpha_{i-1} \,\middle|\,  \eta_1^K(\overline{b}) \right] \\
  \end{aligned}
  \end{equation}
  The first inequality is Markov's inequality. 
  The second inequality is \lemref{lem:compare_probs}.
  We have
  \begin{equation} \label{eq:lem_lower_bound_policy_comp_2}
  \begin{aligned}
    P\left[ \min_{t \le T_i'} x_{k_t,t} \ge \alpha_{i-1} \,\middle|\,  \eta_1^K(\overline{b}) \right] & \ge \left( 1 - \frac{\alpha_i}{K} - \sum_{j=i+1}^{\infty} \alpha_j \right)^{\hspace{-4pt}T_i'} \\
    & \ge \left( 1 - \frac{\alpha_i(1+\frac{1}{i})}{K} \right)^{\hspace{-3pt}T_i'} \\
    & \ge e^{-\alpha_i (1+\frac{1}{i})^2 T_i'/K} \\
    & \ge \alpha_i^{(1+\frac{1}{i})^2c_i } .
  \end{aligned}
  \end{equation}
  The first inequality lower bounds the probability of obtaining a value of $\alpha_i$ or less at every iteration.
  The second inequality uses \lemsubref{lem:alpha_conditions}{ass:3}.
  The third inequality uses \lemref{lem:exp_approx} and \lemsubref{lem:alpha_conditions}{ass:2}.
  The fourth inequality uses the definition $T_i'=\lfloor c_i K \log(1/\alpha_i)/\alpha_i \rfloor$. 
  Combining the \eqref{eq:lem_lower_bound_policy_comp_1} and \eqref{eq:lem_lower_bound_policy_comp_2} gives
  \begin{align*}
    \frac{1}{K} \sum_{k'=1}^K \mathbb E\left[ \min_{t \le T_i'} x_{k_t,t} \,\middle|\,  \mu_1^K(b^{k'}) \right] & \ge \alpha_{i-1} e^{-\frac{2\alpha_i T_i'}{iK}} \alpha_i^{(1+\frac{1}{i})^2c_i }  \\
    & \ge 2\alpha_i^{\frac{1}{i}} \alpha_i^{\frac{2c_i}{i}} \alpha_i^{(1+\frac{1}{i})^2c_i } \\
    & = 2\alpha_i .
  \end{align*}
  The second inequality uses \lemsubref{lem:alpha_conditions}{ass:4} and the definition of $T_i'$.
  The third line uses the definition $c_i=(1-1/i)/((1+1/i)^2+2/i)$, which was chosen to make the third line hold. 
  This completes the proof of the lemma. 
\end{proof}

Noting that \lemref{lem:lower_bound_policy} holds uniformly over the values of $b_j$ for $j \ne i$, a direct consequence of \lemref{lem:lower_bound_policy} is \corref{cor:lower_bound_policy_prob}. 
\begin{corollary} \label{cor:lower_bound_policy_prob}
  We have
  \begin{equation*}
    P_{b \sim D}\left( \mathbb E\left[ \min_{t \le T_i'} x_{k_t,t} \middle|  \mu_1^K(b) \right] \ge 2\alpha_i  \right) \ge \frac{1}{K} ,
  \end{equation*}
  where $D$ is the distribution over $\{1,\ldots,K\}^{\infty}$ defined by sampling each component independently and uniformly at random from $\{1,\ldots,K\}$. 
  The outer probability is over $b$, and the inner expectation is over the $x_{k_t,t}$. 
\end{corollary}

\subsection{Proof of \thmref{thm:main_result}} \label{sec:proof_of_main_result}

Here we synthesize the above results to prove \thmref{thm:main_result}.
\lemref{lem:upper_bound_oracle} and \corref{cor:lower_bound_policy_prob} together imply that
\begin{align*}
  & P_{b \sim D} \left( \mathbb E \left[ \min_{t \le T_i'} x_{k_t,t} \,\middle|\,  \mu_1^K(b) \right] \ge 2\alpha_i > \mathbb E \left[ \min_{t \le T_i} x_{k_*,t} \,\middle|\, \mu_1^K(b) \right] \right) \\
   & \ge \frac{1}{K} ,
\end{align*}
which directly implies that $P(R_{T_i}^{\pi,\pi_*} \ge T_i'/T_i) \ge 1/K$. 
Recall that for a sequence of events $A_i$, we have $P(\text{infinitely many} \, A_i \,\text{happen}) \ge \limsup P(A_i)$. 
This can be seen by applying Fatou's lemma to the relevant indicator functions. 
It follows that
\begin{equation*}
  P_{b \sim D} \left( R_{T_i}^{\pi,\pi_*} \ge \frac{T_i'}{T_i} \, \text{for infinitely many} \, i \right) \ge \frac{1}{K} .
\end{equation*}
Recall the definitions
\begin{equation*}
  T_i=\lceil \log(1/\alpha_i)/\alpha_i \rceil \qquad T_i'=\lfloor c_i K \log(1/\alpha_i)/\alpha_i \rfloor .
\end{equation*}
Since $c_i \to 1$, it follows that $T_i'/T_i \to K$, and so there exists a tuple $\mu_1^K \in M_K$ such that $\limsup_T R_T^{\pi,\pi_*} \ge K$, proving the claim.

\section{Related Work} \label{sec:related_work}
Our setting is closely related to the multi-armed bandit problem, which has been studied extensively. 
See \citet{bubeck2012regret} for a survey.
Regret is the most common measure of performance, though some authors study ``simple regret'' \citep{bubeck2011pure}, where the goal is to identify the arm with the lowest mean.
However, these settings provide little guidance on designing a policy to minimize the single smallest cost. 
The extreme bandit problem, where we care not about the average cost but about the single minimal cost, has been significantly less studied.

The extreme bandit problem (also called the max $K$-armed bandit problem) is introduced in \citet{cicirello2005max} and further developed in \citet{streeter2006asymptotically,streeter2006simple}.
The problem is additionally studied in \citet{carpentier2014extreme}, where the authors give an explicit algorithm and prove that it exhibits asymptotically no regret in the sense of \eqref{eq:old_regret_def}.
However, all results in previous work have relied heavily on strong parametric or semiparametric assumptions on the distributions $\mu_1^K$ under consideration. 
Motivated by extreme value theory, \citet{cicirello2005max} assume that the distributions belong to the Gumbel family and \citet{carpentier2014extreme} consider distributions in the Fr\'echet family (or distributions that are well approximated by the Fr\'echet family).
When the individual samples arise as the maxima of a large number of independent, identically-distributed random variables, then these assumptions may be realistic.
These assumptions dramatically simplify the problem.
As in the multi-armed bandit setting, where every sample from a distribution provides information about the mean of the distribution, in the parametric setting, every sample provides information about the parameters of the distribution.
Once we have accurately estimated each distribution, we can make sensible choices about which distribution to choose.
Our work shows that some form of assumptions are necessary to improve on the guarantees of the policy that chooses each arm equally often. 

We do not expect the parametric assumptions motivated by extreme value theory to make sense in the setting of hyperparameter optimization. 
However, the question of what realistic assumptions are likely to hold in practice for hyperparameter optimization is an important question.

More recently, \citet{david2015max} consider a PAC setting for the extreme bandit problem and prove a lower bound on the sample complexity of algorithms that return an answer within $\epsilon$ of the optimal attainable value with probability $1-\delta$. 

The no free lunch theorems are another form of hardness result in the optimization setting.
\citet{wolpert1997no} show that in a discrete setting, all optimization algorithms that never revisit the same point perform equally well in expectation with respect to the uniform distribution over all possible objectives.

\section{Discussion} \label{sec:discussion}

We have shown that a number of subtleties arise in the extreme bandit setting that are not present in the standard bandit setting. 
These include the fact that there is no well-defined ``best'' arm and the fact that strategies that play multiple arms can outperform oracle strategies that play a single arm. 
We have shown that no policy can be guaranteed to perform asymptotically as well as an oracle that plays the single best arm for a given time horizon. 
This result should not be construed to say that no policy can do better in practice.
Indeed, hyperparameter optimization problems in the real world possess many nice structural properties.
For instance, many hyperparameters have a sweet spot outside of which the algorithm performs poorly.
This suggests that many black-box objectives for hyperparameter optimization may exhibit coordinate-wise quasiconvexity.
Crafting plausible assumptions on the objectives and understanding how they translate into conditions on the induced distributions over algorithm performance is an important problem. 

\subsubsection*{Acknowledgements}
We would like to thank Bal{\'a}zs K{\'e}gl for valuable discussions. 
We would like to thank Kevin Jamieson and Ilya Tolstikhin for their feedback on earlier drafts of this paper. 

\bibliographystyle{abbrvnat}
\bibliography{refs}


\appendix

\section{The Best Arm Depends on the Time Horizon} \label{sec:ex_different_arms}

In \exref{ex:different_arms}, we considered an infinite collection of arms~$\mu_s$ indexed by $0<s<1$.
Samples $x_{s,t}$ from $\mu_s$ satisfy $P(x_{s,t}=s)=s$ and $P(x_{s,t}=1)=1-s$.
We claimed that for a time horizon of $T$, the optimal $s$ is $\Theta((\log T)/T)$. 

We have
\begin{equation*}
  \mathbb E\left[ \min_{ t \le T} x_{s,t} \right] = s (1 - (1 - s)^T) + 1(1-s)^T = s + (1-s)^{T+1} .
\end{equation*}
Let $s_*$ be the index of the optimal distribution, so $\min_s \mathbb E[\min_{ t \le T} x_{s_*,t}] = s_* + (1-s_*)^{T+1}$. 
For large $T$, we can consider the range $0 < s \le \frac12$.
 We have
\begin{equation*}
  s + e^{-2s(T+1)} \le s + (1-s)^{T+1} \le s + e^{-s(T+1)} .
\end{equation*}
It follows that
\begin{align*}
  & \,\, s_* + e^{-2s_*(T+1)} \\
  \le & \,\, \min_s \mathbb E\left[ \min_{ t \le T} x_{s,t} \right] \\
  \le & \,\, \min_s s + e^{-s(T+1)} \\
  \le & \,\, \frac{\log T}{T+1} + \frac{1}{T} \\
  \le & \,\, \frac{2 \log T}{T} .
\end{align*}
Therefore, $s_* \le (2\log T)/T$ and $e^{-2s_*(T+1)}\le (2\log T)/T$. 
The latter implies that
\begin{equation*}
  s_* \ge \frac{-\log 2 - \log \log T + \log T}{2(T+1)} 
\end{equation*}
These results imply that $s_*$ is $\Theta((\log T)/T)$. 

\section{Proof of \lemref{lem:compare_mixture_to_average}}
\label{sec:proof_of_lem_compare_to_mixture}

Here we state and prove \lemref{lem:compare_mixture_to_average}, which is used in the proof of \lemref{lem:compare_probs}.
The goal of \lemref{lem:compare_mixture_to_average} is to show that the probability of a particular sequence of values under the tuple $\mu_1^K(b^{k'})$, when averaged over the possible values of $k'$, is at least as great (up to a constant $c$) as the probability of the same sequence of values under the averaged tuple $\eta_1^K$. 
Since all values other than the values $1$ and $\alpha_i$ have equal probability under all tuples (for $j \ne i$, the value $\alpha_j$ has probability $\alpha_j$ under the $b_j$th element of each tuple), this lemma focuses on the probabilities of the values that equal $1$.
Recall that $\gamma_k(b^{k'})$ is the probability of obtaining a value of $1$ from $\mu_k(b^{k'})$ and $\gamma_k(\overline{b})$ is the probability of obtaining a value of $1$ from $\eta_k$. 

\begin{lemma} \label{lem:compare_mixture_to_average}
  For integers $n_1,\ldots,n_K \ge 0$ such that $n_{k}\le T$, we have
  \begin{equation*}
    \frac{1}{K} \sum_{k'=1}^K \prod_{k=1}^K \gamma_k(b^{k'})^{n_k} \ge c \prod_{k=1}^K \gamma_k(\overline{b})^{n_k} ,
  \end{equation*}
  where $c=e^{- \frac{2\alpha_i T}{iK}}$. 
\end{lemma}
\begin{proof}

  This result nearly follows from Jensen's inequality.
  Indeed, if the function
  \begin{equation*}
    f(c_1,\ldots,c_K) = \prod_{k=1}^K \left( 1 - c_k\alpha_i - \sum_{j=1 \atop j\ne i}^{\infty} \ind[j=k] \alpha_j\right)^{n_k} 
  \end{equation*}
  were convex, then the result would follow from a single application of Jensen's inequality.
  That is, the result with $c=1$ is precisely the statement
  \begin{equation*}
    \frac{f(1,0,\ldots,0) + \cdots + f(0,\ldots,0,1)}{K} \ge f\left(\frac{1}{K},\ldots,\frac{1}{K}\right) .
  \end{equation*}
  Unfortunately, despite the fact that $f$ is the product of convex functions (over the relevant domains), $f$ itself is not convex.
  To circumvent this difficulty, we will approximate each term with the exponential of an affine function, so that the product of approximations remains convex (because the affine functions simply add).
  As our approximation is imperfect, we pick up a penalty in the form of the constant $c$. 
  Let
  \begin{equation*}
    \omega_{k} = 1 - \sum_{j=1 \atop j \ne i}^{\infty} \ind[j=k] \alpha_j 
    \qquad
    \beta_{i,k} = \frac{\alpha_i}{\omega_{k}} ,
  \end{equation*}
  First write
  \begin{equation} \label{eq:lem_compare_mixture_to_average_comp_1}
  \begin{aligned}
    & \,\, \frac{1}{K} \sum_{k'=1}^K \prod_{k=1}^K \gamma_k(b^{k'})^{n_k} \\
    = & \,\, \frac{1}{K} \sum_{k'=1}^K \prod_{k=1}^K (\omega_{k} - \ind[k'=k]\alpha_i)^{n_k}  \\
    = & \,\, \frac{1}{K} \left( \prod_{k=1}^K \omega_k^{n_k} \right) \sum_{k'=1}^K (1 - \beta_{i,k'})^{n_{k'}} .
  \end{aligned}
  \end{equation}
  Note that by \lemsubref{lem:alpha_conditions}{ass:1}, we have $\omega_{k'} \ge \frac12$ and so $\beta_{i,k'} \le 2\alpha_i$. 
  It follows from \lemref{lem:exp_approx} and \lemsubref{lem:alpha_conditions}{ass:2} that we can write
\begin{equation} \label{eq:lem_compare_mixture_to_average_comp_2}
\begin{aligned}
  & \,\, \frac{1}{K}\sum_{k'=1}^K (1 - \beta_{i,k'})^{n_{k'}} \\
  \ge & \,\, \frac{1}{K} \sum_{k'=1}^K e^{-(1+1/i)\beta_{i,k'} n_{k'} } \\
  \ge & \,\, e^{-(1+1/i) \frac{1}{K} \sum_{k'=1}^K \beta_{i,k'} n_{k'} } \\
  \ge & \,\, e^{-\frac{2\alpha_i T}{i K}} e^{- \frac{1}{K} \sum_{k'=1}^K \beta_{i,k'} n_{k'} } \\
  \ge & \,\, e^{-\frac{2\alpha_i T}{i K}} \prod_{k'=1}^K \left( 1 - \frac{\beta_{i,k'}}{K}\right)^{n_{k'}} .
\end{aligned}
\end{equation}
The second inequality is Jensen's inequality. 
The third inequality breaks the $1+1/i$ term into two terms and uses the bounds $\beta_{i,k'} \le 2\alpha_i$ and $n_{k'} \le T$. 
The fourth inequality uses the fact that $e^{-x}\ge 1-x$. 
Combining \eqref{eq:lem_compare_mixture_to_average_comp_1} and \eqref{eq:lem_compare_mixture_to_average_comp_2} gives
\begin{align*}
  & \,\, \frac{1}{K} \sum_{k'=1}^K \prod_{k=1}^K \gamma_k(b^{k'})^{n_k} \\
  \ge & \,\,  e^{-\frac{2\alpha_i T}{i K}} \left( \prod_{k=1}^K \omega_k^{n_k} \right)   \prod_{k'=1}^K \left( 1 - \frac{\beta_{i,k'}}{K}\right)^{n_{k'}} \\
  = & \,\, e^{-\frac{2\alpha_i T}{i K}}  \prod_{k=1}^K \left( \omega_k - \frac{\alpha_i}{K} \right)^{n_k} \\
  = & \,\, e^{-\frac{2\alpha_i T}{i K}}  \prod_{k=1}^K \gamma_k(\overline{b})^{n_k} ,
\end{align*}
which finishes the proof.
\end{proof}

\section{Upper Bound on Exponential}
\label{sec:upper_bound_on_exp}

Throughout this paper, we make use of the inequality $e^{-x} \ge 1-x$. 
However, on a couple of occasions, we need to lower bound $1-x$ by an exponential of the form $e^{-rx}$ for some constant $r$. 
The bound that we use is given in \lemref{lem:exp_approx}. 

\begin{lemma} \label{lem:exp_approx}
  For $i \ge 1$ and $y \in [0,\frac{1}{2(1+i)}]$, we have $e^{-y(1 + \frac{1}{i})} \le 1 - y$.
\end{lemma}
\begin{proof}

  More generally, the convexity of $e^{-x}$ implies that for $0 \le x \le c$, we have
  \begin{equation*}
    e^{-x} \le 1 - \frac{1-e^{-c}}{c}x .
  \end{equation*}
  The right hand side is the formula for the line interpolating between the points $(0,1)$ and $(c,e^{-c})$ on the graph of $e^{-x}$.
  Choosing $c=\log(1+\frac{1}{i})$, and noting that $0 \le x \le \frac{1}{1+i}$ implies that $0 \le x \le c$ because of the standard inequality $1-\frac{1}{x} \le \log x$, we see that $0 \le x \le \frac{1}{1+i}$ implies that
  \begin{equation*}
    e^{-x} \le 1 - \frac{1-\frac{i}{1+i}}{\log(1+\frac{1}{i})}x \le 1 - \frac{\frac{1}{1+i}}{\frac{1}{i}}x = 1 - \frac{i}{1+i}x. 
  \end{equation*}
  Setting $y=\frac{i}{1+i}x$ and using the fact that $\frac{1}{2(1+i)} \le \frac{i}{(1+i)^2}$ gives the result.   
\end{proof}


\end{document}